\documentclass{article}

\usepackage{graphicx,fullpage} 
\usepackage{subfigure}
\usepackage{amsmath,amssymb,amsthm,amsfonts}

\usepackage{natbib}

\usepackage{algorithm}
\usepackage{algorithmic}

\usepackage{hyperref}

\newcommand{\E}{\mathbb{E}}
\newcommand{\R}{\mathbb{R}}
\newcommand{\KL}{\mathop{\mathrm{KL}}}

\newtheorem{thm}{Theorem}

\newtheorem{lemma}{Lemma}

\usepackage[accepted]{icml2013}

\icmltitlerunning{Optimal Rates for Stochastic Convex Optimization}

\begin{document} 

\twocolumn[
\icmltitle{Optimal rates for first-order stochastic convex optimization\\ under 
Tsybakov noise condition}

\icmlauthor{Aaditya Ramdas}{aramdas@cs.cmu.edu}
\icmladdress{Carnegie Mellon University,
            5000 Forbes Ave, Pittsburgh, PA 15213, USA}
\icmlauthor{Aarti Singh}{aarti@cs.cmu.edu}
\icmladdress{Carnegie Mellon University,
            5000 Forbes Ave, Pittsburgh, PA 15213, USA}

\icmlkeywords{machine learning, stochastic convex optimization, minimax rates, generalized uniform convexity, tsybakov noise condition}

\vskip 0.3in
]

\begin{abstract} 
We focus on the problem of minimizing a convex function $f$ over a convex set $S$ given $T$ queries to a stochastic first order oracle. We argue that the complexity of convex minimization is only determined by the rate of growth of the function around its minimizer $x^*_{f,S}$, as quantified by a Tsybakov-like noise condition. Specifically, we prove that if $f$ grows at least as fast as $\|x-x^*_{f,S}\|^\kappa$ around its minimum, for some $\kappa > 1$, then the optimal rate of learning $f(x^*_{f,S})$ is  $\Theta(T^{-\frac{\kappa}{2\kappa-2}})$. The classic rate $\Theta(1/\sqrt T)$ for convex functions and $\Theta(1/T)$ for strongly convex functions are special cases of our result for $\kappa \rightarrow \infty$ and $\kappa=2$, and even faster rates are attained for $\kappa <2$. We also derive tight bounds for the complexity of learning $x_{f,S}^*$, where the optimal rate is $\Theta(T^{-\frac{1}{2\kappa-2}})$. Interestingly, these precise rates for convex optimization also characterize the complexity of active learning and our results further strengthen the connections between the two fields, both of which rely on feedback-driven queries.
\end{abstract}

\section{Introduction and problem setup}
\label{intro}
Stochastic convex optimization in the first order oracle model is the task of approximately minimizing a convex function over a convex set, given oracle access to unbiased estimates of the function and gradient at any point, by using as few queries as possible \cite{NY83}.

A function $f$ is convex on $S$ if, for all $x,y \in S$, $t \in [0,1]$,
\begin{equation*}
f(tx + (1-t)y) \leq tf(x) + (1-t)f(y)
\end{equation*}
$f$ is Lipschitz with constant $L$ if for all $x,y \in S$,
\begin{equation*}
|f(x) - f(y)| \leq L \|x-y\|
\end{equation*}
Equivalently, for subgradients $g_x \in \partial f(x)$, $\|g_x\|_* \leq L$.

Without loss of generality, everywhere in this paper we shall always assume $\|.\| = \|.\|_*=\|.\|_2$, and we shall always deal with convex functions with $L=1$. Furthermore, we will consider 
the set $S\subseteq \R^d$ to be closed bounded convex sets with diameter $D = \max_{x,y\in S} \|x-y\| \leq 1$. Let the collection of all such sets be $\mathbb{S}$. Given $S \in \mathbb{S}$, let the set of all such convex functions on $S$ be $\mathcal{F}^C$ (with $S$ implicit).

A stochastic first order oracle is a function that accepts $x \in S$ as input, and returns $(\hat{f}(x),\hat{g}(x))$ where $\E[\hat{f}(x)] = f(x)$, $\E[\hat{g}(x)]=g(x)$ (and furthermore, they have unit variance)  where $g(x) \in \partial f(x)$ and the expectation is over any internal randomness of the oracle. Let the set of all such oracles be $\mathcal{O}$. As we refer to it later in the paper, we note that a stochastic zeroth order oracle is defined analogously but only returns unbiased function values and no gradient information.

An optimization algorithm is a method $M$ that repeatedly queries the oracle at points in $S$ and returns $\hat{x}_T$ as an estimate of the optimum of $f$ after $T$ queries. Let the set of all such procedures be $\mathcal{M}$. A central question of the field is \textit{``How close can we get to the optimum of a convex function given a budget of $T$ queries?''}.

Let $x^*_{f,S}=\arg \min_{x \in S}f(x)$. Distance of an estimate $\hat{x}_T$ to the optimum $x^*_{f,S}$ can be measured in two ways. We define the \textit{function-error} and \textit{point-error} of $M$ as: 
\begin{equation*} \epsilon_T (M,f,S,O) = f(\hat{x}_T)-f(x^*_{f,S})  
\end{equation*}
\begin{equation*} \rho_T (M,f,S,O) = \|\hat{x}_T - x^*_{f,S}\| 
\end{equation*}

There has been a lot of past work on worst-case bounds for $\epsilon_T$ for common function classes. Formally, let
\begin{equation*} \epsilon^*_T (\mathcal{F}) = \sup_{O \in \mathcal{O}} \sup_{S \in \mathcal{S}} \inf_{M \in \mathcal{M}} \sup_{f \in \mathcal{F}} \E_O[\epsilon_T (M,f,S,O)]
\end{equation*}
\begin{equation*} \rho^*_T (\mathcal{F}) = \sup_{O \in \mathcal{O}} \sup_{S \in \mathcal{S}} \inf_{M \in \mathcal{M}} \sup_{f \in \mathcal{F}} \E_O[\rho_T (M,f,S,O)]
\end{equation*}

It is well known \cite{NY83} that for the set of all convex functions, $\epsilon^*_T (\mathcal{F}^C) = \Theta(1/\sqrt T)$. However, better rates are possible for smaller classes, like that of strongly convex functions, $\mathcal{F^{SC}}$. 

A function $f$ is strongly convex  on $S$ with parameter $\lambda > 0$ if for all  $x,y \in S$ and for all $t \in [0,1]$,
\begin{equation*}
f(tx + (1-t)y) \leq tf(x) + (1-t)f(y) -\frac{1}{2}\lambda t(1-t)\|x-y\|^2
\end{equation*}
Intuitively, this condition means that $f$ is lower bounded by a quadratic everywhere (in contrast, convex functions are lower bounded by a hyperplane everywhere). Again, it is well known \cite{NY83, ABRW10, HK11} that that for the set of all strongly convex functions, $\epsilon^*_T (\mathcal{F}^{SC}) = \Theta(1/ T)$. An immediate geometric question arises - what property of strongly convex functions allows them to be minimized quicker?  

In this work, we answer the above question by characterizing precisely what determines the optimal rate and we derive what exactly that rate is for more general classes. We intuitively describe why such a characterization holds true and what it means by connecting it to a central concept in active learning. These bounds are shown to be tight for both function-error $f(x) - f(x^*_{f,S})$ and the less used, but possibly equally important, point-error $\|x-x^*_{f,S}\|$. 

We claim that the sole determining factor for minimax rates is a condition about the growth of the function only around its optimum, and not a global condition about the strength of its convexity everywhere in space. For strongly convex functions, we get the well-known result that for optimal rates it is sufficient for the function to be lower bounded by a quadratic only around its optimum (not everywhere). 

As we shall see later, any $f \in \mathcal{F}^{SC}$ satisfies
\begin{equation}
f(x) - f(x^*_{f,S}) \geq \frac{\lambda}{2}\|x-x^*_{f,S}\|^2
\end{equation}
On the same note, given a set $S \in \mathbb{S}$, let $\mathcal{F}^\kappa$ represent the set of all convex functions such that for all $x \in S$
\begin{equation}
\label{SCTNC}
f(x) - f(x^*_{f,S}) \geq \frac{\lambda}{2}\|x-x^*_{f,S}\|^\kappa
\end{equation}
for some $\kappa \geq 1$. This forms a nested hierarchy of classes of $\mathcal{F}^C$, with $\mathcal{F}^{\kappa_1} \subset \mathcal{F}^{\kappa_2}$ whenever $\kappa_1 < \kappa_2$. Also notice that $\mathcal{F}^2 \supseteq \mathcal{F}^{SC}$ and $\bigcup_\kappa \mathcal{F}^\kappa \subseteq \mathcal{F}^C$. For any finite $\kappa < \infty$, this condition automatically ensures that the function is strictly convex and hence the minimizer is well-defined and unique.

Then we can state our main result as:

\begin{thm}
\label{main}
Let $\mathcal{F}^\kappa$ ($\kappa > 1$) be the set of all $1$-Lipschitz convex functions on $S \in \mathbb{S}$ satisfying $f(x) - f(x^*_{f,S}) \geq \frac{\lambda}{2}\|x-x_{f,S}^*\|^\kappa$ for all $x \in S$ for some $\lambda > 0$. Then, for first order oracles, we have $\epsilon^*_T (\mathcal{F}^\kappa) = \Theta(T^{-\frac{\kappa}{2\kappa-2}})$ and $\rho^*_T (\mathcal{F}^\kappa)=\Theta(T^{-\frac{1}{2\kappa-2}})$. Also, for zeroth order oracles, we have $\epsilon^*_T (\mathcal{F}^\kappa) = \Omega(1/\sqrt T)$ and $\rho^*_T (\mathcal{F}^\kappa)=\Omega(T^{-\frac{1}{2\kappa}})$.
\end{thm}

Note that for $\epsilon^*_T$ we get faster rates than $1/T$ for $\kappa < 2$. For example, if we choose $\kappa = 3/2$, then we surprisingly get $\epsilon^*_T(\mathcal{F}^{3/2})=\Theta(T^{-3/2})$.

The proof idea in the lower bound arises from recognizing that the growth condition in equation (\ref{SCTNC}) closely resembles the Tsybakov noise condition (TNC) \footnote{Sometimes goes by Tsybakov margin/regularity condition \cite{KT93,T09}} from statistical learning literature, which is known to determine minimax rates for passive and active classification \cite{T09,CN07} and level set estimation \cite{T97,S09}. 

Specifically, we modify a proof from \cite{CN07} that was originally used to find the minimax lower bound for active classification where the TNC was satisfied at the decision boundary. We translate this to our setting to get a lower bound on the optimization rate, where the function satisfies a convexity strength condition at its optimum. One can think of the rate of growth of the function around its minimum as determining how much the oracle's noise will drown out the true gradient information, thus measuring the signal to noise ratio near the optimum.

\cite{RR09} notice that stochastic convex optimization and active learning have similar flavors because of the role of feedback and sequential dependence of queries. 
Our results make this connection more precise by demonstrating that the complexity of convex optimization in d-dimensions
is precisely the same as the complexity of active learning in 1 dimension.
Specifically, the rates we derive for 
function error and point error in first-order stochastic convex 
optimization of a d-dimensional function are precisely the same as the rates for 
classification error and error in localizing the decision boundary, respectively, in 
1-dimensional active learning \cite{CN07}. 

This result agrees with intuition since in 
1 dimension, finding the
decision boundary and the minimizer are equivalent to finding the zero-crossing of the regression function, $P(Y|X=x)-1/2$, or the zero-point of the gradient, respectively (see Section~\ref{conn1d} for details). Thus in 1D, it requires the same number of 
samples or time steps to find the
decision boundary or the minimizer, respectively, using feedback-driven
queries. 
In higher dimensions, the decision boundary becomes a multi-dimensional set whereas, for a
convex function, the minimizer continues to be the point of zero-crossing of the gradient.
Thus, rates for active learning degrade exponentially in dimension, whereas rates for 
first-order stochastic convex optimization don't.

For upper bounds, we slightly alter a recent variant of gradient descent from \cite{HK11} and prove that it achieves the lower bound. While there exist algorithms in passive (non-active) learning that achieve the minimax rate without knowing the true behaviour at the decision boundary, unfortunately our upper bounds depend on knowing the optimal $\kappa$.


\subsection{Summary of contributions}
\begin{itemize}
\item We provide an interesting connection between strong convexity (more generally, uniform convexity) and the Tsybakov Noise Condition which is popular in statistical learning theory \cite{T09}. Both can be interpreted as the amount by which the signal to noise ratio decays on approaching the minimum in optimization or the decision boundary in classification. 
\item We use the above connection to strengthen the relationship between the fields of active learning and convex optimization, the seeds of which were sown in \cite{RR09} by showing that the rates for first-order stochastic convex optimization of a $d$-dimensional function are precisely the rates for $1$-dimensional active learning. 
\item Using proof techniques from active learning \cite{CN07}, we get lower bounds for a hierarchy of function classes $\mathcal{F}^\kappa$, generalising known results for convex, strongly convex \cite{NY83}, \cite{ABRW10} and uniformly convex classes \cite{ST10}. 
\item We show that the above rates are tight (all $\kappa > 1$) by generalising an algorithm from \cite{HK11} that was known to be optimal for strongly convex functions, and also reproduce the optimal rates for $\kappa$-uniformly convex functions (only defined for $\kappa \geq 2$) \cite{JN10}. 
\item Our lower bounding proof technique also gets us, for free, lower bounds for the derivative free stochastic zeroth-order oracle setting, a generalization of those derived in \cite{JNR12}. 
\end{itemize}

\section{From Uniform Convexity to TNC}

A function $f$ is said to be $\kappa$-uniformly convex ($\kappa \geq 2$) on $S \in \mathbb{S}$ if, for all $x,y \in S$ and all $t \in [0,1]$,
$$f(tx + (1-t)y) \leq tf(x) + (1-t)f(y) - \frac{1}{2}\lambda t(1-t)\|x-y\|^\kappa$$
for some $\lambda>0$ \cite{JN10}. 

An equivalent first-order condition, is that for any subgradient $g_x \in \partial f(x)$, we have for all $x,y \in S$,
\begin{equation}
\label{SC}
f(y) \geq f(x) + g_x^\top(y-x) + \frac{\lambda}{2}\|y-x\|^\kappa
\end{equation}
When $\kappa=2$, this is well known as strong convexity. It is well known that since $0 \in \partial f(x^*_{f,S})$, we have for all $x \in S$,
\begin{equation}
f(x) \geq f(x^*_{f,S}) + \frac{\lambda}{2}\|x-x^*_{f,S}\|^\kappa
\end{equation}
This local condition is strictly weaker than (\ref{SC}) and it only states that the function grows at least as fast as $\|x-x^*_{f,S}\|^\kappa$ around its optimum. This bears a striking resemblance to the Tsybakov Noise Condition (also called the regularity or margin condition) from the statistical learning literature.

\paragraph{Tysbakov's Noise Condition} We reproduce a relevant version of the condition from \cite{CN07}. Define $\eta(x) := P(\ell(x)=1|x)$, where $\ell(x)$ is the label of point $x$. Let $x^*$ be the closest point to $x$ such that $\eta(x^*) = 1/2$, ie on the decision boundary. $\eta$ is said to satisfy the TNC with exponent $\kappa \geq 1$ if 
\begin{equation}
|\eta (x) - \eta (x^*)| \geq \lambda \|x-x^*\|^\kappa
\end{equation}
for all $x$ in such that $|\eta (x) - 1/2| \leq \delta$ with $\delta > 0$. 

It is natural to conjecture that the strength of convexity and the TNC play similar roles in determining minimax rates, and that rates of optimizing functions should really \textit{only} depend on a TNC-like condition around their minima, motivating the definition of $\mathcal{F}^\kappa$ in equation \ref{SCTNC}. We emphasize that though uniform convexity is not defined for $\kappa < 2$, $\mathcal{F}^\kappa$ is well-defined for $\kappa \geq 1$ (see Appendix, Lemma 1).

The connection of the strength of convexity around the optimum to TNC is very direct in one-dimension, and we shall now see that it enables us to use an active classification algorithm to do stochastic convex optimization.

\subsection{Making it transparent in 1-D}
\label{conn1d}

We show how to reduce the task of stochastically optimizing a one-dimensional convex function to that of active classification of signs of a monotone gradient. For simplicity of exposition, we assume that the set $S$ of interest is $[0,1]$, and $f$ achieves a unique minimizer $x^*$ inside the set $(0,1)$.

Since $f$ is convex, its true gradient $g$ is an increasing function of $x$ that is negative before $x^*$ and positive after $x^*$. Assume that the oracle returns gradient values corrupted by unit variance gaussian noise \footnote{The gaussian assumption is only for this subsection}. Hence, one can think of $sign(g(x))$ as being the true label of point $x$, $sign(g(x)+z)$ as being the observed label, and finding $x^*$ as learning the decision boundary (the point where labels switch signs). If we think of $\eta(x) = P(sign(g(x)+z) = 1|x)$, then minimizing $f$ corresponds to identifying the Bayes classifier $[x^*,1]$ because the point at which $\eta(x)=0.5$ is where $g(x)=0$, which is $x^*$. 

If $f(x) - f(x^*) \geq \lambda \|x-x^*\|^\kappa$, then $|g_x| \geq \lambda \|x-x^*\|^{\kappa-1}$(see Appendix, Lemma 2). Let us consider a point $x$ which is a distance $t>0$ to the right of $x^*$ and hence has label $1$ (similar argument for $x < x^*$). 

So, for all $g_x \in \partial f(x), \ g_x \geq \lambda t^{\kappa-1}$. In the presence of gaussian noise $z$, the probability of seeing label $1$ is the probability that we draw $z$ in $(-g_x,\infty)$ so that the sign of $g_x+z$ is still positive. This yields:
\begin{equation*}
\eta(x) \ \ =\ \  P(g_x + z > 0) \ \ =\ \  0.5 + P(-g_x < z < 0) 
\end{equation*}
Note that the probability mass of a gaussian grows linearly around its mean (Appendix, Lemma 3); ie, for all $t < \sigma$ there exist constants $a_1,a_2$ such that $a_1t \leq P(0 \leq z \leq t) \leq a_2t$. So, we get
\begin{eqnarray}
\eta(x) \ \ \geq \ \ \ 0.5 + a_1\lambda t^{\kappa-1} \nonumber \\
\implies \ \ \ \ \ \   |\eta(x) - 1/2| \geq a_1\lambda|x-x^*|^{\kappa-1} \label{bz}
\end{eqnarray}
Hence, $\eta(x)$ satisfies TNC with exponent $\kappa-1$.

\cite{CN07} provide an analysis of the Burnashev-Zigangirov (BZ) algorithm, which is a noise-tolerant variant of binary bisection, when the regression function $\eta(x)$ obeys a TNC like in equation \ref{bz}. The BZ algorithm solves the one-dimensional active classification problem such that after making $T$ queries for a noisy label, it returns a confidence interval $\hat{I}_T$ which contains $x^*$ with high probability, and $\hat{x}_T$ is chosen to be the midpoint of $\hat{I}_T$. They bound the excess risk $\int_{[x,1]\Delta[x^*,1]}|2\eta(x)-1|dx$ where $\Delta$ is the symmetric difference operator over sets but small modifications to their proofs (see Appendix, Lemma 4) yield a bound on $\E|\hat{x}_T - x^*|$.

The setting of $\kappa =1$ is easy because the regression function is bounded away from half (the true gradient doesn't approach zero, so the noisy gradient is still probably the correct sign) and we can show an exponential convergence of  $\E(|\hat{x}_T-x^*|) = O(e^{-T\lambda^2/2})$. The unbounded noise setting of $\kappa > 1$ is harder and using a variant of BZ analysed in \cite{CN07}, we can show (see Appendix, Lemma 5) that $\E(|\hat{x}_T-x^*|) = \tilde{O}\left(\frac{1}{T}\right)^{\frac{1}{2\kappa - 2}}$ and $\E(|\hat{x}_T-x^*|^\kappa) = \tilde{O}\left(\frac{1}{T}\right)^{\frac{\kappa}{2\kappa - 2}}$. \footnote{We use $\tilde{O}$ to hide polylogarithmic factors.}

Interestingly, in the next section on lower bounds, we show that for any dimension, $\Omega \left(\frac{1}{T}\right)^{\frac{1}{2\kappa - 2}}$ is the minimax convergence rate for $\E(\|\hat{x}_T-x^*\|)$.

\section{Lower bounds using TNC} \label{lower}

We prove lower bounds for $\epsilon_T^*(\mathcal{F}^\kappa), \rho^*_T(\mathcal{F}^\kappa)$ using a technique that was originally for proving lower bounds for active classification under the TNC \cite{CN07}, providing a nice connection between active learning and stochastic convex optimization. 
\begin{thm}
\label{LB}
Let $\mathcal{F}^\kappa$ ($\kappa > 1$) be the set of all $1$-Lipschitz convex functions on $S \in \mathbb{S}$ satisfying $f(x) - f(x^*_{f,S}) \geq \frac{\lambda}{2}\|x-x_{f,S}^*\|^\kappa$ for all $x \in S$ for some $\lambda > 0$. Then, we have $\epsilon^*_T (\mathcal{F}^\kappa) = \Omega(T^{-\frac{\kappa}{2\kappa-2}})$ and $\rho^*_T (\mathcal{F}^\kappa)=\Omega(T^{-\frac{1}{2\kappa-2}})$.
\end{thm}
The proof technique is summarised below. We demonstrate an oracle $O^*$ and set $S^*$ over which we prove a lower bound for $\inf_{M \in \mathcal{M}}\sup_{f \in \mathcal{F}^\kappa} \E_O[\epsilon_T(M,f,S,O)]$. Specifically, let $S^*$ be $[0,1]^d \cap \{\|x\|\leq 1\}$ and $O^*$ just adds standard normal noise to the true function and gradient values. We then pick two similar functions in the class $\mathcal{F}^\kappa$ and show that they are hard to differentiate with only $T$ queries to  $O^*$.

We go about this by defining a semi-distance between any two elements of $\mathcal{F}^\kappa$ as the distance between their minima. We then choose two very similar functions $f_0, f_1$ whose minima are $2a$ apart (we shall fix $a$ later). The oracle chooses one of these two functions and the learner gets to query at points $x$ in domain $S^*$, receiving noisy gradient and function values $y \in \R^d,z \in \R$. We then define distributions corresponding to the two functions $P^0_T, P^1_T$ and choose $a$ so that these distributions are at most a constant KL-distance $\gamma$ apart. We then use Fano's inequality which, using $a$ and $\gamma$, lower bounds the probability of identifying the wrong function by any estimator (and hence optimizing the wrong function) given a finite time horizon of length $T$.

The use of Fano's inequality is not new to convex optimization, but proofs that lower-bound the probability of error under a sequential, feedback-driven querying strategy are prominent in active learning, and we show such proofs also apply to convex optimization thanks to the relation of uniform convexity around the minimum to the Tysbakov Noise Condition. We state Fano's inequality for completeness:

\begin{thm} \cite{T09}
\label{Fano}
Let $\mathcal{F}$ be a model class with an associated semi-distance $\delta(\cdot,\cdot ) : \mathcal{F} \times \mathcal{F} \rightarrow \R$ and each $f \in \mathcal{F}$ having an associated measure $P^f$ on a common probability space. Let $f_0,f_1 \in \mathcal{F}$ be such that $\delta(f_0,f_1) \geq 2a > 0$ and $KL(P^0 || P^1) \leq \gamma$. Then, 
\begin{eqnarray*}
\inf_{\hat{f}} \sup_{f\in\mathcal{F}} P^f \left(\delta(\hat{f},f) \geq a\right)  \geq \max \left(\frac{\exp (-\gamma)}{4}, \frac{1 - \sqrt {\gamma/2}}{2}\right)
\end{eqnarray*}
\end{thm}

\subsection{Proof of Theorem \ref{LB}}
For technical reasons, we choose a subclass $\mathcal{U}^\kappa \subset \mathcal{F}^\kappa$  which is chosen such that every point in $S^*$ is the unique minimizer of exactly one function in $\mathcal{U}^\kappa$. By construction of $\mathcal{U}^\kappa$, returning an estimate $\hat{x}_T \in S^*$ is equivalent to identifying the function $\hat{f}_T \in \mathcal{U}^\kappa$ whose minimizer is at $\hat{x}_T$. So we now proceed to bound $\inf_{\hat{f}_T}\sup_{f \in \mathcal{U}^\kappa} \E \|\hat{x_T} - x_{f,S^*}^*\|$.

Recall that we chose $S^* = [0,1]^d \cap \{\|x\|\leq 1\}$. Define the semi-distance $\delta(f_a,f_b) = \|x^*_a - x^*_b\|$ and let \footnote{For $\kappa=2$, note that $f_0,f_1 \in \mathcal{F}^{SC}$ (strongly convex)}
$$
f_0(x) = c_1 \sum_{i=1}^d |x_i|^\kappa = c_1 \|x\|_\kappa^\kappa
$$
$$
g_0(x) = \kappa c_1 (x_1^{\kappa-1},...,x_d^{\kappa-1})
$$

so that $x^*_{0,S^*} = \vec{0}$. Now define $\vec{a_1} = (a,0,...,0)$ and let
$$
f_1(x) = \left\{ \begin{array}{ll}
  c_1 \left(\|x-2\vec{a_1}\|_\kappa^\kappa + c_2 \right)   &\mbox{ $x_1 \leq 4a$} \\
  f_0(x) &\mbox{ o.w.}\\
       \end{array} \right.\\
$$
$$
g_1(x) = \left\{ \begin{array}{ll}
  \kappa c_1 \left (\frac{|x_1 - 2a|^{\kappa}}{(x_1-2a)}, x_2^{\kappa-1},...,x_d^{\kappa-1} \right) &\mbox{ $x_1 \leq 4a$} \\
   g_0(x) &\mbox{ o.w.}\\
       \end{array} \right.\\
$$

so that $x^*_{1,S^*} = 2\vec{a}$ and hence $\delta(f_0,f_1) = 2a$. Notice that these two functions and their gradients differ only on a set of size $4a$.  Here, $c_2 = (4a)^\kappa - (2a)^\kappa$ is a constant ensuring that $f_2$ is continuous at $x_1=4a$, and $c_1$ is a constant depending on $\kappa,d$ ensuring that the functions are $1$-Lipschitz on $S^*$. Both parts of $f_1$ are convex and the gradient of $f_1$ increases from $x_1=4a^-$ to $x_1=4a^+$, maintaining convexity. Hence we conclude that both functions are indeed convex and both are in $\mathcal{F}^\kappa$ for appropriate $c_1$ (Appendix, Lemma 6). Our interest here is the dependence on $T$, so we ignore these constants to enhance readability.

On querying at point $X=x$, the oracle returns $Z \sim \mathcal{N}(f(x),\sigma^2))$ and $Y \sim \mathcal{N}(g(x),\sigma^2I_d)$. In other words, for $i =0,1$, we have $P^i(Z_t,Y_t|X=x_t) = \mathcal{N}\left((f_i(x_t),g_i(x_t)),\sigma^2I_{d+1}\right)$. Let $S_1^T = (X_1^T, Y_1^T, Z_1^T)$ be the set of random variables corresponding to the whole sequence of $T$ query points and responses. Define a probability distribution corresponding to every $f \in \mathcal{U}^\kappa$ as the joint distribution of $S_1^T$ if the true function was $f$, and so
$$P^0_T := P^0(X_1^T, Y_1^T, Z_1^T), \ \ P^1_T := P^1(X_1^T, Y_1^T, Z_1^T)$$

We show that the KL-divergence of these distributions is $\KL(P_T^0,P_T^1) = O(Ta^{2\kappa -2})$ and choose $a = T^{-\frac{1}{2\kappa-2}}$ so that $\KL(P_T^0,P_T^1) \leq \gamma$ for some constant $\gamma > 0$.

\begin{lemma}
$\KL(P^0_T, P^1_T) = O(Ta^{2\kappa-2})$ 
\end{lemma}
\begin{proof}
\vspace{-4 mm}
\begin{eqnarray*}
&& \KL (P^0_T,P^1_T)  = \E^0 \left[ \log \frac{P^0(X_1^T, Y_1^T, Z_1^T)}{P^1(X_1^T, Y_1^T, Z_1^T)} \right] \nonumber \\
&& \hspace{-15 mm}= \E^0 \left[ \log \frac{\prod_t P^0(Y_t,Z_t|X_t) P(X_t|X_1^{t-1},Y_1^{t-1},Z_1^{t-1}) }{\prod_t P^1(Y_t,Z_t|X_t) P(X_t|X_1^{t-1},Y_1^{t-1},Z_1^{t-1})} \right] \label{past}\\
&=& \E^0 \left[ \log \frac{\prod_{t=1}^T P^0(Y_t,Z_t|X_t) }{\prod_{t=1}^T P^1(Y_t,Z_t|X_t) } \right] \nonumber\\
&=& \sum_{t=1}^T E^0 \left[ \E^0 \left[ \log \frac{P^0(Y_t,Z_t|X_t) }{P^1(Y_t,Z_t|X_t) } \Bigg\vert X_1,...,X_T \right] \right] \nonumber\\
&\leq& T \max_{x \in [0,1]^d} \E^0 \left[ \log \frac{P^0(Y_1,Z_1|X_1) }{P^1(Y_1,Z_1|X_1) } \Bigg\vert X_1=x \right] \nonumber \\
&& \hspace{-10mm}= T \max_{x \in [0,1]^d} \E^0 \left[ \log \frac{P^0(Y_1|X_1)P^0(Z_1|X_1) }{P^1(Y_1|X_1)P^1(Z_1|X_1) } \Bigg\vert X_1=x \right] \label{indep}\\
&\leq& T \left( \max_{x \in [0,1]^d} \E^0 \left[ \log \frac{P^0(Y_1|X_1) }{P^1(Y_1|X_1) } \Bigg\vert X_1=x \right] \right) \nonumber \\ 
&& + T \left( \max_{x \in [0,1]^d} \E^0 \left[ \log \frac{P^0(Z_1|X_1) }{P^1(Z_1|X_1) } \Bigg\vert X_1=x \right] \right) \nonumber \\
&=& \frac{T}{2} \left(  \max_{x \in [0,1]^d} \|g_0(x)-g_1(x)\|^2 \right) \nonumber \\
&& + \frac{T}{2} \left( \max_{x \in [0,1]^d} (f_0(x)-f_1(x))^2 \right) \label{normKL}\\
\end{eqnarray*}
\begin{eqnarray*}
&=& \frac{c_1^2T}{2} \left( \kappa^2 \max_{x_1 \in [0,4a]} \left( \frac{|x_1 - 2a|^{\kappa}}{(x_1-2a)} - x_1^{\kappa-1} \right)^2 \right) \nonumber \\
&& + \frac{c_1^2T}{2} \left( \max_{x_1 \in [0,4a]} ( |x_1 - 2a|^{\kappa} - x_1^{\kappa} )^2 \right) \label{subst}\\
&=& O(Ta^{2\kappa-2}) + O(Ta^{2\kappa}) = O(Ta^{2\kappa-2}) \nonumber
\end{eqnarray*}
(\ref{past}) follows because the distribution of $X_t$ conditional on $X_1^{t-1},Y_1^{t-1},Z_1^{t-1}$ depends only on the algorithm $M$ and does not change with the underlying distribution. (\ref{indep}) follows because $Y_t \perp Z_t$ when conditioned on $X_t$. We also used $(Y_i,Z_i|X_i) \perp (Y_j,Z_j|X_j)$ for $i \neq j$. (\ref{normKL}) follows because the KL-divergence between two identity-covariance gaussians is just half the squared euclidean distance between their means. (\ref{subst}) follows by simply substituting the gradient/function values which differ only on $x_1\in [0,4a]$. 
\end{proof}

Using Theorem \ref{Fano} with $a=T^{-\frac{1}{2\kappa-2}}$, for some $C > 0$ we get $\inf_{\hat{f}_T} \sup_{f \in  \mathcal{U}^\kappa} P_f(\delta(\hat{f}_T,f) \geq a) \geq C$. Hence,
\begin{eqnarray*}
\inf_{\hat{f}_T} \sup_{f \in \mathcal{U}^\kappa} \E \|\hat{x}_T - x^*_f\|  \geq  a \cdot \inf_{\hat{f}_T} \sup_{f \in  \mathcal{U}^\kappa} P_f(\delta(\hat{f}_T,f) 
\geq a)\\  
\geq \ \ \  a \cdot C \ \ \ = \ \ \ C T^{-\frac{1}{2\kappa-2}}
\end{eqnarray*}

where we used Markov's inequality, Fano's inequality and finally the aforementioned choice of $a$. 

This gives us our required bound on $\rho_T^*(\mathcal{U}^\kappa)$, and correspondingly also for $\epsilon_T^*(\mathcal{U}^\kappa)$ because
\begin{eqnarray*}
\inf_M\sup_{f \in \mathcal{U}^\kappa} \E [f(\hat{x_T}) - f(x_f^*)]  \geq   \inf_M\sup_{f \in \mathcal{U}^\kappa} \lambda [\E \|\hat{x_T} - x^*_f\| ^\kappa] \\ \ \ \ \geq  \ \ \ \inf_{\hat{f}_T}\sup_{f \in \mathcal{U}^\kappa} \lambda[\E \|\hat{x_T} - x^*\|]^\kappa\ \ \ 
\end{eqnarray*}
where the first inequality follows because $f \in \mathcal{F}^\kappa$, and the second follows by applying Jensen's for $\kappa > 1$. Finally, we get the bounds on $\rho^*_T(\mathcal{F}^\kappa)$ and $\epsilon^*_T(\mathcal{F}^\kappa)$ because we are now taking $\sup$ over the larger class $\mathcal{F}^\kappa \supset \mathcal{U}^\kappa$. This concludes the proof of Theorem \ref{LB}.

This is a generalisation of known lower bounds, because we can recover existing lower bounds for the convex and strongly convex settings by choosing $\kappa \rightarrow \infty$ and $\kappa=2$ respectively. Furthermore, we will show that these bounds are tight for all $\kappa > 1$. These bounds also immediately yield lower bounds for uniformly convex functions, since $\|x\|_\kappa^\kappa$ is $\kappa$-uniformly convex (Appendix, Lemma 8) which can also be arrived from the results of \cite{ST10} using an online-to-batch conversion.

\subsection{Derivative-Free Lower Bounds}

The above proof immediately gives us a generalization of recent tight lower bounds for derivative free optimization \cite{JNR12}, in which the authors consider zeroth-order oracles (no gradient information) and find that $\epsilon^*_T(\mathcal{F}^C)= \Theta(1/\sqrt T) = \epsilon^*_T(\mathcal{F}^{SC})$ \footnote{The $\kappa$ in \cite{JNR12} should not be confused with our TNC exponent $\kappa=2$ for $\mathcal{F}^{SC}$} concluding that strong convexity does not help in this setting. Here, we show

\begin{thm}
Let $\mathcal{F}^\kappa$ ($\kappa > 1$) be the set of all $1$-Lipschitz convex functions on $S \in \mathbb{S}$ satisfying $f(x) - f(x^*_{f,S}) \geq \frac{\lambda}{2}\|x-x_{f,S}^*\|^\kappa$ for all $x \in S$ for some $\lambda > 0$. Then, in the derivative-free zeroth-order oracle setting, we have $\epsilon^*_T (\mathcal{F}^\kappa) = \Omega(1/\sqrt T)$ and $\rho^*_T (\mathcal{F}^\kappa)=\Omega(T^{-\frac{1}{2\kappa}})$.
\end{thm}

Ignoring $y$, $Y_1^T$, define $P^0_T := P^0(X_1^T, Z_1^T), P^1_T := P^1(X_1^T, Z_1^T)$ to get $\KL(P_T^0,P_T^1) = O(Ta^{2\kappa})$. Choose $a = T^{-\frac{1}{2\kappa}}$ so that $\KL(P_T^0,P_T^1) \leq \gamma$ for some $\gamma > 0$, and apply Fano's to get $\inf_{\hat{f}_T} \sup_{f \in \mathcal{U}^\kappa} \E \|\hat{x}_T - x^*_f\| = C T^{-\frac{1}{2\kappa}}$ for some $C>0$. 


\section{Upper Bounds using Epoch-GD}
\label{upper}

\begin{algorithm}[t]
 \caption{EpochGD (domain $S$, exponent $\kappa > 0$, convexity parameter $\lambda > 0$, confidence $\delta > 0$, oracle budget $T$, subgradient bound $G$)}
Initialize $x_1^1 \in S$ arbitrarily, $e=1$\\
Initialize $T_1 = 2C_0$, $\eta_1 = C_1 \ 2^{-\frac{\kappa}{2\kappa-2}}, R_1= \left(\frac{C_2\eta_1}{\lambda}\right)^{1/\kappa}$\\
 \begin{algorithmic}[1]
   \WHILE{$\sum_{i=1}^{e} T_i \leq T$}
   \FOR {$t=1$ to $T_e$}
   \STATE Query the oracle at $x_t^e$ to obtain $\hat{g}_t$
   \STATE $$ x_{t+1}^e = \prod_{S \cap B(x^e_1,R_e)} (x_t^e - \eta_e \hat{g}_t) $$
        \ENDFOR
   \STATE Set $x_1^{e+1} = \frac{1}{T_e} \sum^{T_e}_{t=1} x_t^e$
   \STATE Set $T_{e+1} = 2 T_e$, $\eta_{e+1} = \eta_e \cdot 2^{-\frac{\kappa}{2\kappa-2}}$ 
   \STATE Set $R_{e+1} = \left(\frac{C_2\eta_{e+1}}{\lambda}\right)^{1/\kappa}$, $e \leftarrow e+1$
   \ENDWHILE
 \end{algorithmic}
 \textbf{Output:} $x^e_1$ \\
 \label{EGD}
\end{algorithm}

We show that the bounds from Secton \ref{lower} are tight by presenting an algorithm achieving the same rate.

\begin{thm}
\label{UB}
Algorithm $EpochGD(S,\kappa,T,\delta,G,\lambda)$ returns $\hat{x}_T \in S$ after $T$ queries to any oracle $O \in \mathcal{O}$, such that for any $f \in \mathcal{F}^\kappa, \kappa>1$ on any $S \in \mathbb{S}$, $f(\hat{x}_T) - f(x^*_f) = \widetilde{O}(T^{-\frac{\kappa}{2\kappa-2}})$ and $\|\hat{x}_T - x^*_f\| = \widetilde{O}(T^{-\frac{1}{2\kappa-2}})$ hold with probability at least $1 - \delta$ for any $\delta > 0$. \footnote{$\widetilde{O}$ hides $\log\log T$ and $\log(1/\delta)$ factors}
\end{thm}

Recall that for $f \in \mathcal{F}^\kappa$, $\|g_x\| \leq 1$ for any subgradient at any $x \in S$. Since the oracle may introduce bounded variance noise, we have $\|\hat{g}_x\| \leq 1 + c\sigma^2$ with high probability. Here, to keep a parallel with \cite{HK11}, we use $\|\hat{g}_x\| \leq G$ for convenience. Also, in algorithm \ref{EGD} $B(x,R)$ refers to the ball around $x$ of radius $R$ i.e. $B(x,R) = \{y \ |\  \|x-y\| \leq R\}$. 

We note that for uniformly convex functions ($\kappa \geq 2$), \cite{JN10} derive the same upper bounds. Our rates are valid for $1 < \kappa < 2$ and hold more generally as we have a weaker condition on $\mathcal{F}^\kappa$.

\subsection{Proof of Theorem \ref{UB}}
We generalize the proof in \cite{HK11} for strongly convex functions ($\kappa=2$) and derive values for $C_0,C_1$ and $C_2$ for which Theorem $\ref{UB}$ holds. We begin by showing that $f$ having a bounded subgradient corresponds to a bound on the diameter of $S$, and hence on the maximum achievable function value.

\begin{lemma}
\label{max}
If $f \in \mathcal{F}^\kappa$ and $\|g_x\| \leq G$, then for all $x \in S$, we have $\|x - x_f^*\| \leq (G\lambda^{-1})^{\frac{1}{\kappa-1}} =: D$ and $f(x) - f(x_f^*) \leq (G^\kappa \lambda^{-1})^{\frac{1}{\kappa-1}} =: M$ 
\end{lemma}\vspace{-0.2in}
\begin{proof}
By convexity, $f(x) - f(x_f^*) \leq g_x^\top (x-x_f^*) \leq \|g_x\| \cdot \|x-x^*_f\|$ (Holder's inequality), implying that $G\|x-x_f^*\| \geq f(x)-f(x_f^*) \geq \lambda \|x-x^*\|^\kappa$.

Hence,  $\|x-x_f^*\|^{\kappa -1} \leq G/\lambda$ or $\|x - x_f^*\| \leq  G^{\frac{1}{\kappa-1}}/\lambda^{\frac{1}{\kappa-1}}$. Finally $f(x) - f(x_f^*) \leq G\|x - x_f^*\| \leq G^{\frac{\kappa}{\kappa-1}}/\lambda^{\frac{1}{\kappa-1}}$. 
\end{proof}

\begin{lemma} \label{whp}
Let $\|x_1 - x_f^*\| \leq R$. Apply $T$ iterations of the update $x_{t+1} = \Pi_{S \cap B(x_1, R)} (x_t - \eta \hat{g}_t)$, where $\hat{g}_t$ is an unbiased estimator for the subgradient of $f$ at $x_t$ satisfying $\|\hat{g}_t\|\leq G$. Then for $\bar{x} = \frac{1}{T} \sum_t x_t$ and any $\delta > 0$, with probability at least $1 - \delta$, we have 
$$
f(\bar{x}) - f(x_f^*) \leq \frac{\eta G^2}{2} + \frac{\|x_1 - x_f^*\|^2}{2\eta T} + \frac{4GR\sqrt{2\log (1/\delta)}}{\sqrt T}
$$
\end{lemma}
\begin{proof} Lemma 10 in \cite{HK11}.
 \end{proof}


\begin{lemma}
For any epoch $e$ and any $\delta > 0$, $T_e = C_0 2^e$, $E = \lfloor \log (\frac{T}{C_0} + 1) \rfloor$, $\eta_e = C_12^{-e \frac{\kappa}{2\kappa -2}}$, for appropriate $C_0,C_1,C_2$, we have with probability at least $(1 - \frac{\delta}{E})^{e-1}$ 
$$ \Delta_e := f(x_1^e) - f(x^*_f) \leq C_2 \eta_e$$
\end{lemma}
\vspace{-0.1in}
\begin{proof}
We let $\widetilde{\delta} = \frac{\delta}{E}$ and use proof by induction on $e$.

The first step of induction, $e=1$, requires $$\Delta_1 \leq C_2 \eta_1 = C_2C_12^{-\frac{\kappa}{2\kappa-2}} \ \ \ \textbf{[R1]}$$
\vspace{-0.2in}

Assume that $\Delta_e \leq C_2 \eta_e$ for some $e\geq 1$, with probability at least $(1-\widetilde{\delta})^{e-1}$ and we now prove it correspondingly for epoch $e+1$. We condition on the event  $\Delta_e \leq C_2 \eta_e$ which happens with the above probability. By the TNC, $\Delta_e \geq \lambda\|x_1^e - x^*\|^\kappa$, and the conditioning implies that $\|x_1^e - x^*\| \leq (C_2\eta_e/\lambda)^{1/\kappa}$, which is the radius $R_e$ of the ball for the EpochGD projection step.

Lemma \ref{whp} applies with $R = R_e= (\frac{C_2\eta_e}{\lambda})^{\frac{1}{\kappa}}$ and so with probability at least $1 - \widetilde{\delta}$, we have 
$$\Delta_{e+1} \leq \frac{\eta_e G^2}{2} + \frac{\|x^e_1 - x^*\|^2}{2\eta_e T_e} + \frac{4G (\frac{C_2\eta_e}{\lambda})^{\frac{1}{\kappa}} \sqrt{2\log (\frac{1}{\widetilde{\delta}})}}{\sqrt T_e} 
$$
$$\leq \frac{\eta_e G^2}{2} + \frac{C_2^{\frac{2}{\kappa}}\eta_e^{\frac{2}{\kappa}}}{2\eta_e T_e \lambda^{\frac{2}{\kappa}}} + \frac{4G (\frac{C_2\eta_e}{\lambda})^{\frac{1}{\kappa}} \sqrt{2\log (\frac{1}{\widetilde{\delta}})}}{\sqrt T_e}$$

For the induction, we would like $RHS \leq \eta_eG^2 \leq C_2\eta_{e+1}$ which can be achieved by
$$\frac{C_2^{\frac{2}{\kappa}}\eta_e^{\frac{2}{\kappa}}}{2\eta_e T_e \lambda^{\frac{2}{\kappa}}} \leq \frac{\eta_eG^2}{6} \ \ \ \textbf{[R2]}$$
$$\frac{4G (\frac{C_2\eta_e}{\lambda})^{\frac{1}{\kappa}} \sqrt{2\log (\frac{1}{\widetilde{\delta}})}}{\sqrt T_e} \leq \frac{\eta_eG^2}{3} \ \ \ \textbf{[R3]}$$  
$$\eta_eG^2 \leq C_2\eta_{e+1} \ \ \  \textbf{[R4]}$$

Then, factoring in the conditioned event which happens with probability at least $(1-\widetilde{\delta})^{e-1}$ we would get $\Delta_{e+1} \leq C_2\eta_{e+1}$ with probability at least $(1-\widetilde{\delta})^e$. 

We set $C_0,C_1,C_2$ such that the four conditions hold.

$\textbf{[R4]} \implies C_2 \geq G^22^{\frac{\kappa}{2\kappa-2}}$, a lower bound for $C_2$. 

$\textbf{[R2]} \implies C_1 \geq \left( \frac{3}{G^2C_0} \right)^{\frac{\kappa}{2\kappa-2}} \left(\frac{C_2}{\lambda}\right)^{\frac{1}{\kappa-1}}$ 

$\textbf{[R3]} \implies C_1 \geq \left( \frac{3(96 \log(1/\widetilde{\delta}))}{G^2C_0} \right)^{\frac{\kappa}{2\kappa-2}} \left(\frac{C_2}{\lambda}\right)^{\frac{1}{\kappa-1}}$  

This is the stronger condition on $C_1$.

Observe that if $C_0 = 288 \log(1/\widetilde{\delta})$, by substitution we get the inequality  $C_2\eta_1 = C_1C_22^{-\frac{\kappa}{2\kappa-2}} \geq M 2^{\frac{\kappa}{2(\kappa-1)^2}}$

\textbf{[R1]} is trivially true for the above choices of $C_0,C_1,C_2$, because $\Delta_1 \leq M \leq M 2^{\frac{\kappa}{2(\kappa-1)^2}} \leq C_2\eta_1$

Hence, $C_0 = 288 \log(E/\delta), \ C_1 = \frac{G^{\frac{2-\kappa}{\kappa-1}} 2^{\frac{\kappa}{2(\kappa-1)^2}}}{\lambda^{\frac{1}{\kappa-1}}}$ and $C_2=G^22^{\frac{\kappa}{2\kappa-2}}$ satisfy the lemma. As a sanity check, \cite{HK11} choose $C_0 = 288 \log(E/\delta), C_1 = 2/\lambda, C_2 = 2G^2$ for strongly convex functions.
\end{proof}
The algorithm runs for $E=\lfloor \log (\frac{T}{C_0} + 1) \rfloor$ rounds so that the total number of queries is at most $T$. \footnote{We lose $\log \log T$ factors here, like \cite{HK11}. Alternatively, using $E=\lfloor \log (\frac{T}{288} + 1) \rfloor$, we could run for $T\log\log T$ steps and get error bound $O(T^{-\frac{\kappa}{2\kappa-2}})$.} The bound for $\Delta_{E+1}$ yields the bounds on function error immediately by noting that $(1-\frac{\delta}{E})^E \geq 1 - \delta$ and since $f \in \mathcal{F}^\kappa$, we can bound the point error
$$\|\hat{x}_T - x^*\| \leq \lambda^{-1/\kappa}[f(\hat{x}_T) - f(x^*)]^{1/\kappa}$$

\vspace{-0.35in}
\section{Discussion and future work}

The most common assumptions in the literature for proving convergence results for optimization algorithms are those of convexity and strong convexity, and \cite{JN10} recently prove upper bounds using dual averaging for $\kappa$-uniformly convex functions when $\kappa \geq 2$. These classes impose a condition on the behaviour of the function, the strength of its convexity, everywhere in the domain. The TNC condition for our smooth hierarchy of classes is natural and strictly weaker because it is implied by uniform convexity or strong convexity in the realm of $\kappa \geq 2$, and has no corresponding notion when $1 < \kappa < 2$.

The lower bound $\Omega(T^{-\frac{\kappa}{2\kappa-2}})$ for $\epsilon^*$ that we prove immediately gives us the $\Omega(1/T)$ lower bound for strongly convex functions and the classic $\Omega(1/\sqrt T)$ bound when $\kappa \rightarrow \infty$. The lower bound $\Omega(T^{-\frac{1}{2\kappa-2}})$ for $\rho^*$ is interesting because the optimization literature does not often focus on point-error estimates. We demonstrate how to use an active learning proof technique that is novel in its application to optimization, having the additional benefit that it also gives tight rates for derivative free optimization with no additional work. 
It is useful to have a unified proof generalizing rates for convex, strongly convex, uniformly convex and more in both the first and zeroth order stochastic oracle settings.

The rates for both $\epsilon^*$ and $\rho^*$ are strongly supported by intuition as seen by the rate's behaviour at the extremes of $\kappa$. 
$\kappa \rightarrow 1$ is the best case because of large signal to noise ratio, as the gradient jumps signs rapidly without spending time around zero where it can be corrupted by noise, and we should be able to identify the optimum extremely fast (function error rates better than $1/T$), as supported by our result for the bounded noise setting in 1-D and the tight upper bounds using Epoch-GD. However, when $\kappa \rightarrow \infty$, the function is extremely flat around its minimum, and while we can optimize function-error well (because a lot of points have function value close to the minimum), it is hard to get close to the minimizer with noisy samples.

Our upper bounds on $\epsilon$ and $\rho$ involve a generalization of Epoch Gradient Descent \cite{HK11}, and demonstrate that the lower bounds achieved in terms of $\kappa$ are correct and tight. We make the same assumptions as \cite{JN10} and \cite{HK11} - number of time steps $T$, a bound on noisy subgradients $G$ and the convexity parameter $\lambda$. Substituting $\kappa=2$ in our algorithm yields the $O(1/T)$ rate for strongly convex functions and $\kappa \rightarrow \infty$ recovers the $O(1/\sqrt T)$ rate for convex functions.

Our lower bound proof bounds $\epsilon^*$ and $\rho^*$ simultaneously, by bounding point-error and using the class definition to bound function-error (for both first and zeroth order oracles). The upper-bound proofs proceed in the opposite direction by bounding function-error and then using TNC condition to bound point-error.

In practice, one may not know the degree of convexity of the function at hand, but every function has a unique smallest $\kappa$ for which it is in $\mathcal{F}^\kappa$, and using a larger $\kappa$ will still maintain convergence (but at slower rates). If we only know that $f$ is convex then we can use any gradient descent algorithm, and if we know it is strongly convex then we can use $\kappa=2$, so our algorithm is not any weaker than existing ones, but it is certainly stronger if we know $\kappa$ exactly.

Designing an algorithm which is adaptive to unknown $\kappa$ is an open problem. Function and gradient values should enable characterization of the function in a region, but a function may have different smoothness is different parts of the space and old gradient information could be misleading. For example, consider a function on $[-0.5,0.5]$ which is $2x^2$ between $[-0.25,0.25]$, and grows linearly with gradient $\pm 1$ elsewhere. This function is not strongly convex, but it is in $\mathcal{F}^2$, and it changes behaviour at $x=\pm 0.25$.

Hints of connections to active learning have been lingering in the literature, as noted by \cite{RR09}, but our borrowed lower bound proof from active learning and the one-dimensional upper bound reduction from stochastic optimization to active learning gives hope of a much more fertile intersection. While many active learning methods degrade exponentially with dimension $d$, the rates in optimization degrade polynomially since active learning is trying to solve harder problem like learning a $(d-1)$-dimensional decision boundary or level set, while optimization problems are just interested in getting to a single good point (for any $d$). 
This still leaves open the possibility of using a one dimensional active learning algorithm as a subroutine for a $d$-dimensional convex optimization problem, or a generic reduction from one setting to the other (given an algorithm for active learning, can it solve an instance of stochastic optimization). It is an open problem to prove a positive or negative result of this type. 
We feel that this is the start of stronger conceptual ties between these fields.

\section{Acknowledgements}
This research is supported in part by AFOSR grant FA9550-10-1-0382 and NSF grant IIS-1116458. We thank Sivaraman Balakrishnan, Martin Wainwright, Alekh Agarwal, Rob Nowak and reviewers for inputs.

\bibliography{ICML2013}
\bibliographystyle{icml2013}

\section*{Appendix}

\section*{Section 2}

\begin{lemma} 
No function can satisfy Uniform Convexity for $\kappa < 2$, but they can be in $\mathcal{F}^\kappa$ for $\kappa < 2$. 
\end{lemma}
\begin{proof}
If uniform convexity could be satisfied for (say) $\kappa=1.5$, then we have for all $x,y \in S$
$$
f(y) - f(x) - g_x^\top (y-x) \geq \frac{\lambda}{2}\|x-y\|_2^{1.5}
$$
Take $x,y$ both on the positive $\mathbf{x}$-axis. The Taylor expansion would require, for some $c \in [x,y]$,
$$
f(y) - f(x) - g_x^\top (y-x) \ \ \ = \ \ \  \frac{1}{2}(x-y)^\top H(c) (x-y) $$  
$$\ \ \ \leq \ \ \ \frac{\|H(c)\|_F}{2} \|x-y\|_2^2
$$
Now, taking $\|x-y\|_2 = \epsilon \rightarrow 0$ by choosing $x$ closer to $y$, the Taylor condition requires the residual to grow like $\epsilon^2$ (going to zero fast), but the UC condition requires the residual to grow at least as fast as $\epsilon^{1.5}$ (going to zero slow). At some small enough value of $\epsilon$, this would not be possible. Since the definition of UC needs to hold for all $x,y \in S$, this gives us a contradiction. So, no $f$ can be uniformly convex for any $\kappa < 2$

However, one can note that for $f(x) = \|x\|_{1.5}^{1.5} = \sum_i |x_i|^{1.5}$, we have $x_f^* = 0$, and $f(x) - f(x_f^*) = \|x\|_{1.5}^{1.5} \geq \|x - x_f^*\|_2^{1.5}$, hence $f \in \mathcal{F}^{1.5}$.\\
\end{proof}

\begin{lemma}
If $f \in \mathcal{F}^\kappa$, then for any subgradient $g_x \in \partial f(x)$, we have $\|g_x\|_2 \geq \lambda \|x-x^*\|_2^{\kappa-1}$.
\end{lemma}
\begin{proof}
By convexity, we have 
$f(x^*) \geq f(x) + g_x^\top (x^* - x)$. 
Rearranging terms and since $f \in \mathcal{F}^\kappa$, we get 
$$g_x^\top (x - x^*) \geq f(x) - f(x^*) \geq \lambda \|x-x^*\|_2^\kappa$$
By Holder's inequality, 
$$\|g_x\|_2 \|x-x^*\|_2 \geq g_x^\top (x - x^*)$$
Putting them together, we have 
$$\|g_x\|_2 \|x-x^*\|_2 \geq \lambda \|x-x^*\|_2^\kappa$$ 
giving us our result.\\
\end{proof}

\begin{lemma}
For a gaussian random variable $z$, $\forall t < \sigma, \ \ \ \exists a_1,a_2, \ \ \ a_1t \leq P(0 \leq z \leq t) \leq a_2t$
\end{lemma}

\begin{proof} We wish to characterize how the probability mass of a gaussian random variable grows just around its mean. 
Our claim is that it grows linearly with the distance from the mean, and the following simple argument argues this neatly. 

Consider a $X \sim N(0,\sigma^2)$ random variable at a distance $t$ from the mean $0$. 
We want to bound $\int_{-t}^t d\mu(X)$ for small $t$. 
The key idea in bounding this integral is to approximate it by a smaller and larger rectangle, each 
having a width $2t$ (from $-t$ to $t$).

The first one has a height equal to $\frac{e^{-t^2/2\sigma^2}}{\sigma \sqrt{2\pi}}$, the smallest value taken by the gaussian in $[-t,t]$ achieved at $t$, and the other with a height equal to the $\frac{1}{\sigma \sqrt{2\pi}}$, the largest value of the gaussian in $[-t,t]$ achieved at 1.

The smaller rectangle has area $2t \frac{e^{-t^2/2\sigma^2}}{\sigma \sqrt{2\pi}} \geq 2t \frac{e^{-1/2}}{\sigma \sqrt{2\pi}}$ when $t<\sigma$. The larger rectangle clearly has an area of $2t \frac{1}{\sigma \sqrt{2\pi}}$.

Hence we have $A_1t = 2t \frac{1}{\sigma \sqrt{2\pi e}} \leq P(|X| < t) \leq 2t \frac{1}{\sigma \sqrt{2\pi}} = A_2t$ for $t < \sigma$. 
Similarly, for a one-sided inequality, we have $a_1t = t \frac{1}{\sigma \sqrt{2\pi e}} \leq P(0 <X< t) \leq t \frac{1}{\sigma \sqrt{2\pi}} = a_2t$ for $t < \sigma$.

We note that the gaussian tail inequality $P(X > t) \leq \frac{1}{t}e^{-t^2/2\sigma^2}$ really makes sense for large $t>\sigma$ and we are interested in $t<\sigma$.
There are tighter inequalities, but for our purpose, this will suffice. \\
\end{proof}

\begin{lemma}
If $|\eta(x) -1/2| \geq \lambda$, the midpoint $\hat{x}_T$ of the high-probability interval returned by BZ satisfies $\E|\hat{x}_T - x^*| = O(e^{-T\lambda^2/2})$. \cite{CN07}
\end{lemma}
\begin{proof}
The BZ algorithm works by dividing $[0,1]$ into a grid of $m$ points (interval size $1/m$) and makes $T$ queries (only at gridpoints) to return an interval $\hat{I}_T$ such that $\Pr(x^* \notin \hat{I}_T) \leq me^{-T\lambda^2}$ \cite{CN07}. We choose $\hat{x}_T$ to be the midpoint of this interval, and hence get
\begin{eqnarray*}
&& \E|\hat{x}_T - x^*| = \int_0^1 \Pr (|\hat{x}_T - x^*|>u)du\\
&=& \int_0^{1/2m} \Pr (|\hat{x}_T - x^*|>u)du \\
&& + \int_{1/2m}^1 \Pr (|\hat{x}_T - x^*|>u)du \\
&\leq& \frac{1}{2m} + \left(1- \frac{1}{2m}\right) \Pr \left(|\hat{x}_T - x^*|> \frac{1}{2m} \right)\\
&\leq& \frac{1}{2m} + me^{-T\lambda^2} = O\left(e^{-T\lambda^2/2}\right)\\
\end{eqnarray*}
for the choice of the number of gridpoints as $m = e^{T\lambda^2/2}$.\\
\end{proof}

\begin{lemma}
If $|\eta(x) -1/2| \geq \lambda |x-x^*|^\kappa$, the point $\hat{x}_T$ obtained from a modified version of BZ satisfies $\E|\hat{x}_T - x^*| = O\left((\frac{\log T}{T})^{\frac{1}{2\kappa-2}}\right)$ and $\E[|\hat{x}_T - x^*|^\kappa] = O\left((\frac{\log T}{T})^{\frac{\kappa}{2\kappa-2}}\right)$.
\end{lemma}
\begin{proof}
We again follow the same proof as in \cite{CN07}. Initially, they assume that the grid points are not aligned with $x^*$, ie $\forall k \in \{0,...,m\}, \ \ \ |x^* - k/m| \geq 1/3m$. This implies that for all gridpoints $x$, $|\eta(x) - 1/2| \geq \lambda (1/3m)^{\kappa-1}$. Following the exact same proof above,
\begin{eqnarray*}
&& \E[|\hat{x}_T - x^*|^\kappa] = \int_0^1 \Pr (|\hat{x}_T - x^*|^\kappa>u)du\\
&=& \int_0^{(1/2m)^\kappa} \Pr (|\hat{x}_T - x^*|>u^{1/\kappa})du \\
&& + \int_{(1/2m)^\kappa}^1 \Pr (|\hat{x}_T - x^*|>u^{1/\kappa})du \\
&\leq& \left(\frac{1}{2m} \right)^\kappa + \left(1- \left(\frac{1}{2m} \right)^\kappa \right) \Pr \left(|\hat{x}_T - x^*|> \frac{1}{2m} \right)\\
&\leq& \left(\frac{1}{2m} \right)^\kappa + m\exp (-T\lambda^2 (1/3m)^{2\kappa-2}) \\
&&= O\left(\left(\frac{T}{\log T} \right)^{\frac{1}{2\kappa -2}} \right)
\end{eqnarray*}
on choosing $m$ proportional to $\left(\frac{T}{\log T} \right)^{\frac{1}{2\kappa -2}}$. 

\cite{CN07} elaborate in detail how to avoid the assumption that the grid points don't align with $x^*$. They use a more complicated variant of BZ with three interlocked grids, and gets the same rate as above without that assumption. The reader is directed to their exposition for clarification.\\
\end{proof}

\section*{Section 3}

\begin{lemma}
$ c_\kappa \|x\|_\kappa^\kappa = c_\kappa \sum_{i=1}^d |x_i|^\kappa =: f_0(x) \in \mathcal{F}^\kappa$, for all $\kappa > 1$. Also, $f_1(x)$ as defined in Section 3 is also in $\mathcal{F}^\kappa$.
\end{lemma}
\begin{proof}
Firstly, this is clearly convex for $\kappa > 1$. Also, $f_0(x^*_{f_0})=0$ at $x^*_{f_0}=0$. So, all we need to show is that for appropriate choice of $c_\kappa$, $f$ is indeed $1$-Lipschitz and that $f_0(x)-f_0(x^*_{f_0}) \geq \lambda \|x-x^*_{f_0}\|_2^\kappa$ for some $\lambda > 0$, ie
$$c_\kappa \|x\|_\kappa^\kappa \geq \lambda \|x\|_2^\kappa \ \ \ , \ \ \ c_\kappa (\|x\|_\kappa^\kappa - \|y\|_\kappa^\kappa) \leq \|x-y\|_2$$

Let us consider two cases, $\kappa \geq 2$ and $\kappa < 2$. Note that all norms are uniformly bounded with respect to each other, upto constants depending on $d$. Precisely, if $\kappa < 2$, then $\|x\|_\kappa > \|x\|_2$ and if $\kappa \geq 2$, then $\|x\|_\kappa \geq d^{1/\kappa - 1/2}\|x\|_2$.

When $\kappa \geq 2$, consider $c_\kappa=1$. Then 
$$(\|x\|_\kappa^\kappa - \|y\|_\kappa^\kappa) \leq \|x-y\|_\kappa^\kappa \leq \|x-y\|_2^\kappa \leq \|x-y\|_2$$
because $\|z\|_\kappa \leq \|z\|_2$ and $\|x-y\| \leq 1$.
Also, $\|x\|_\kappa^\kappa \geq d^{1-\frac{\kappa}{2}} \|x\|_2^\kappa$, so $\lambda = d^{1-\frac{\kappa}{2}}$ works.

When $\kappa < 2$, consider $c_\kappa=\frac{1}{\sqrt{d}^\kappa}$. Similarly
 $$c_\kappa (\|x\|_\kappa^\kappa - \|y\|_\kappa^\kappa) \leq \left(\frac{\|x-y\|_\kappa}{\sqrt d}\right)^\kappa \leq \|x-y\|_2^\kappa \leq \|x-y\|_2$$
Also $c_\kappa\|x\|^\kappa_\kappa \geq c_\kappa \|x\|^\kappa_2$, so $\lambda = c_\kappa$ works.

Hence $f_0(x)$ is $1$-Lipschitz and in $\mathcal{F}^\kappa$ for appropriate $c_\kappa$.



Now, look at $f_1(x)$ for $x_1 \leq 4a$. It is actually just $f_0(x)$, but translated by $2a$ in direction $x_1$, with a constant added, and hence has the same growth around its minimum. Now, the part with $x_1 > 4a$ is just $f_0(x)$ itself, which have the same growth parameters as the part with $x_1 \leq 4a$. So $f_1(x) \in \mathcal{F}^\kappa$ also.\\
\end{proof}

\begin{lemma}\label{decomp} For all $i=1...d$, let $f_i(x)$ be any one-dimensional $\kappa$-uniformly convex function ($\kappa \geq 2$) with constant $\lambda_i$. For a $d-$dimensional function $f(x) = \sum_{i=1}^d f_i(x_i)$ that decomposes over dimensions, $f(x)$ is also $\kappa$-uniformly convex with constant $\lambda = \frac{\min_i \lambda_i}{d^{1/2 - 1/\kappa}}$.
\end{lemma}

\begin{proof}
\begin{eqnarray*}
&& f(x+h) = \sum_i f_i(x_i+h_i) \\
& \geq & \sum_i ( f_i(x_i) + g_{x_i}h_i + \lambda_i |h_i|^\kappa)
\end{eqnarray*}
\begin{eqnarray*}
& \geq & f(x) + g_x^\top h + (\min_i \lambda_i) \|h\|_\kappa^\kappa\\
& \geq & f(x) + g_x^\top h + \frac{(\min_i \lambda_i)}{d^{1/2-1/\kappa}} \|h\|_2^\kappa\\
\end{eqnarray*}

(one can use $h=y-x$ for the usual first-order definition)

\end{proof}

\begin{lemma}\label{UC} $f(x) = |x|^k$ is $\kappa$-uniformly convex i.e.
$$
tf(x) + (1-t)f(y) \geq f(tx+(1-t)y) + \frac{\lambda}{2} t(1-t) |x-y|^k
$$
for $\lambda = 4/2^k$. Lemma \ref{decomp} implies $\|x\|_\kappa^\kappa$ is also $\kappa$-uniformly convex with $\lambda = \frac{4/2^k}{d^{1/2-1/\kappa}}$.
\end{lemma}

\begin{proof} First we will show this for the special case of $t = 1/2$. 
We need to argue that:
\[
\frac1{2}|x|^k + \frac1{2}|y|^k \geq |\frac{x+y}{2}|^k + \lambda \frac1{8} |x-y|^k
\]
Let $\lambda = 4/2^k$. We will prove a stronger claim -
\[
\frac1{2}|x|^k + \frac1{2}|y|^k \geq |\frac{x+y}{2}|^k + 2\lambda \frac1{8} |x-y|^k
\]

 Since $k\geq 2$
\begin{eqnarray*}
RHS^{1/k} & = &  (|\frac{x+y}{2}|^k + |\frac{x-y}{2}|^k)^{1/k}\\
& \leq & (|\frac{x+y}{2}|^2 + |\frac{x-y}{2}|^2)^{1/2}\\
& \leq & (|x|^2/2 + |y|^2/2)^{1/2}\\
& \leq & \frac1{\sqrt{2}} 2^{1/2-1/k} (|x|^k + |y|^k)^{1/k}\\
& \leq & (\frac1{2}|x|^k + \frac1{2}|y|^k)^{1/k} = LHS^{1/k}
\end{eqnarray*}

Now, for the general case. We will argue that just proving the above for $t=1/2$ is 
actually sufficient.
\begin{eqnarray*}
&& f(tx+(1-t)y) = f\left(2t \left(\frac{x+y}{2}\right) + (1-2t)y\right)\\
& \leq & 2t f\left(\frac{x+y}{2}\right) + (1-2t) f(y)\\
& \leq & tf(x) + tf(y) - 2t\frac{2\lambda}{8}|x-y|^k + (1-2t)f(y)\\
& \leq & tf(x) + (1-t)f(y) -t(1-t)\frac{\lambda}{2}|x-y|^k
\end{eqnarray*}
\end{proof}

\end{document}